\documentclass[a4paper, times, 10pt,twocolumn]{article}
\usepackage[top=4.9cm,bottom=3.7cm,left=1.5cm,right=1.5cm]{geometry}
\usepackage{amsmath,amsthm,amssymb,amsfonts}
\usepackage{graphicx}
\usepackage{textcomp}
\usepackage{multirow}
\usepackage{hyperref}
\usepackage{booktabs}
\usepackage{algorithm}
\usepackage{algpseudocode}
\usepackage{tabularx}
\usepackage{subcaption}
\usepackage[square,numbers]{natbib}
\usepackage{lipsum} 

\usepackage{ICMLC}
\usepackage{times}
\usepackage{graphicx}
\usepackage{indentfirst}
\usepackage{latexsym}
\usepackage[margin=8pt,font=footnotesize,labelfont=bf,labelsep=period
]{caption}

\pdfpagewidth=\paperwidth
\pdfpageheight=\paperheight
\newtheorem{lemma}{Lemma}

\begin{document}
\title{Tiny, Hardware-Independent, Compression-based Classification}  

\author{\bf{\normalsize{C. Meyers${^{1,2}}$, A. P. MacSween, E. Elmroth${^2}$, and~T. L\"{o}fstedt}${^2}$}\\ 
\\
\normalsize{$^1$Institute for Experiential AI, Northeastern University, Boston, United States}\\
\normalsize{$^2$Department of Computer Science, Umeå University, Umeå, Sweden} \\
\normalsize{E-MAIL:\@ c.meyers@cs.umu.se}\\
\\}

\maketitle \thispagestyle{empty}

\maketitle

\begin{abstract}
The recent developments in machine learning have highlighted a conflict between online platforms and their users in terms of privacy.
The importance of user privacy and the struggle for power over user data has been intensified as regulators and operators attempt to police online platforms.
As users have become increasingly aware of privacy issues, client-side data storage, management, and analysis have become a favoured approach to large-scale centralised machine learning.
However, state-of-the-art machine learning methods require vast amounts of labelled user data, making them unsuitable for models that reside client-side and only have access to a single user's data.
State-of-the-art methods are also computationally expensive, which degrades the user experience on compute-limited hardware and also reduces battery life.
A recent alternative approach has proven remarkably successful in classification tasks across a wide variety of data---using a compression-based distance measure (called normalised compression distance) to measure the distance between generic objects in classical distance-based machine learning methods.
In this work, we demonstrate that the normalised compression distance is actually not a metric; develop it for the wider context of kernel methods to allow modelling of complex data; and present techniques to improve the training time of models that use this distance measure.
We demonstrate that the normalised compression distance works as well as and sometimes better than other metrics and kernels---while requiring only marginally more computational costs and in spite of the lack of formal metric properties.
The end results is a simple model with remarkable accuracy even when trained on a very small number of samples allowing for models that are small and effective enough to run entirely on a client device using only user-supplied data.
\end{abstract}



\section{Introduction}

Modern machine learning (ML) methods have demonstrated remarkable efficacy across many domains.
However, they often have large numbers of parameters, and thus also require large numbers of samples to train on~\cite{desislavov2021compute}.
This aggregation of vast amounts of data creates numerous privacy, safety, and security threats~\cite{chat_control} that are consequences of large-scale user data collection, for instance by online platform operators (see Section~\ref{threat} for more details).
We evaluate and extend prior work on compression-based distance measures by proposing to incorporate them into kernel methods, enabling efficient classification even with limited training samples.
When using small training sets and small and efficient models, they can be trained entirely on a client device without sharing private user data with anyone---allowing the model builder to circumvent many weaknesses associated with state-of-the-art methods~\cite{chat_control,power_consumption_ai,desislavov2021compute}.
We demonstrate the efficacy of the proposed approach in the context of malware detection, network intrusion detection, and spam detection.
In addition, we show that our proposed method reduces the run-time relative to the baseline by a factor of approximately 50\% while simultaneously improving the accuracy significantly.

\subsection{Threat Model}
\label{threat}

In the context of online platforms, data are collected from end-users using dubious amounts of consent~\cite{nouwens2020dark} and aggregated at massive scales~\cite{desislavov2021compute}.
Such data collection on online platforms often creates privacy, safety, and security risks~\cite{chakraborty_adversarial_2018,meyers2024massively}.
One such privacy risk is the periodic attempts by regulators to weaken encryption standards~\cite{amnesty_encryption} and create backdoors to user devices.
Platform operators and governments discuss best-practices for scanning client devices for illegal or ``offensive''
content~\cite{chat_control,apple_csam}, but existing proposals are no less risky to user privacy, safety, or security than weakening encryption~\cite{chat_control}.
Privacy experts have denounced such approaches for numerous reasons: the ease of \textit{extracting} private training data~\cite{li2021membership} or the model itself~\cite{orekondy2019knockoff,correia2018copycat}, the ability of a malicious user to induce false positives for other users when the model is trained on user data~\cite[\textit{poisoning} attacks;][]{rawat2022devil,shokri2020bypassing,souri2022sleeper}, and the triviality of simply \textit{evading} the detection mechanism~\cite{carlini_towards_2017,dohmatob_generalized_2019,hopskipjump,chakraborty_adversarial_2018,trashfire}.

Current platform solutions often rely on large-scale ML methods that are trained on vast amounts of user data and federated across devices~\cite{apple_csam}---an approach that has been criticised by privacy experts~\cite{chat_control}.
While, in some cases, personally identifiable information can be removed, applications like spam detection IP address, user name, and message contents are precisely the data we need to classify spam and shut down botnets.
Beyond the inherent privacy concerns, examples of security risks include attacks against ML systems that target a model during training~\cite{biggio_poisoning_2013}, prediction~\cite{carlini_towards_2017}, and deployment~\cite{distributed_attacks}.
Even when access to a model by an adversary is limited, it is possible to induce a misclassification~\cite{hopskipjump}, reverse engineer the model~\cite{extraction_attack}, determine the model weights~\cite{jagielski2020high}, or infer the class-membership of new samples~\cite{bentley2020quantifying}.
This raises profound questions for safety-critical systems~\cite{meyers} and legal questions about access and control of the underlying data~\cite{marks2023ai}.
Even if the attacker only has access to a typical application programming interface (API), there are reliable ways to fool the model~\cite{hopskipjump}.

Many privacy experts warn of the potential for any hypothetical \textit{centralised} content-censoring system not only because of the potential failures due to adversaries mentioned above, but also because of the potential these systems to be used for mass surveillance and censorship~\cite{chat_control}.
In short, distributed and centralised training paradigms are both inherently fragile to malicious users and dangerous for user privacy, even if the resulting model lives on the user device (rather than in the cloud).

\subsection{Motivations}

In contrast to many state-of-the-art methods, this work proposes a light-weight, client-side approach to content filtering that does not rely on large-scale data collection.

Recently, Jiang \textit{et al.}~\cite{jiang2022less} proposed a remarkably successful approach to ``parameter free'' classification, dubbed NCD-KNN, that exploits a compress\-ion-based distance measure, the \textit{normalised compression distance}~\citep[NCD;][]{ncd}, to classify objects using the $k$-nearest neighbours (KNN) method~\cite{shalev2014understanding}.
NCD-KNN is known to work well even when trained on small numbers of samples~\cite{scilipoti2024strong}.
By building a model that is accurate on a small number of samples, we can fulfil the goal of training a ML model entirely on a client device using data generated by a single user.
An additional goal of this work was to evaluate the efficacy of NCD in general and to extend it to kernel methods in particular.

While other research examined topics like image classification~\cite{opitz2023gzip}, molecular property classification~\cite{weinreich2023parameter}, and text classification~\cite{nishida2011tweet}, the ability of NCD to classify datasets that contain strings, numeric values, and categorical data (heterogenous datasets) has remained unexplored.

Additionally, the original NCD work~\cite{ncd} included an error term that is usually ignored in recent research~\cite{opitz2023gzip,weinreich2023parameter,nishida2011tweet,jiang2022less}.

The original authors~\cite{ncd} asserted that NCD is always positive when using perfect compressors.
However, this is clearly demonstrated to be false in Lemma~\ref{pseudometric} when using imperfect compressors (which is what we will have in practice)~\cite{ncd_pitfalls}.

Prior works about NCD have primarily focused on distance-based ML methods (\textit{e.g.}, KNN), which limits both the class of methods that can be considered and the types of data that can be effectively analysed.
A key motivation for this work was to extend the use of NCD beyond distance-based methods by developing a novel kernel-based formulation. This significantly broadens the applicability of NCD, making it suitable for a wider range of machine learning techniques where traditional distance-based approaches cannot be used.

\subsection{Contributions}

To use NCD, the model builder must choose a compression algorithm.
While the effect of various compressors has, in part, been explored before~\cite{ncd_pitfalls}, we expand the analysis by Cebri\'{a}n \textit{et al.}~\cite{ncd_pitfalls} to more recent compression algorithms and also offer additional run-time improvements over previous implementations~\cite{jiang2022less}.

The NCD-KNN method has shown very strong performance across several benchmarks, but prior implementations~\cite{jiang2022less} are not appropriate for real-time settings due to redundant computations.
We therefore propose several modifications in Section~\ref{improvements}.

Further, we show that NCD is not a metric~\cite{opitz2023gzip,weinreich2023parameter,nishida2011tweet,jiang2022less,ncd_pitfalls}, which means that applying ML methods blindly can lead to erroneous results (\textit{e.g.}, by incorrectly ordering the nearest neighbours).
In this work we demonstrate this non-metric behaviour in Lemma~\ref{pseudometric} and propose techniques to mitigate the effects of this behaviour in Section~\ref{improvements}, effectively making the modified NCD ``more like a metric''.

Additionally, we expand the notion of NCD~\cite{opitz2023gzip,weinreich2023parameter,nishida2011tweet,ncd,jiang2022less} to kernels (Section~\ref{kernels}) thus allowing for this method to be used with other models besides KNN\@.
We thus extend NCD to reproducing kernel Hilbert spaces and hence more elaborate ML methods---allowing its use in a broader set of ML methods and to model more complex decision boundaries.

\section{Background}

In the sections below, we define and describe the NCD, outline the NCD-KNN method proposed by Jiang \textit{et al.}~\cite{jiang2022less}, outline several other string metrics, and discuss the NCD distance matrix and how to efficiently compute it.

\subsection{Normalised Compression Distance}
\label{ncd}
NCD has been demonstrated to be a \textit{universal} measure of similarity between two objects~\cite{ncd}---where a value of 0 denotes equivalence and a value of 1 denotes complete dissimilarity.
The NCD is defined as~\cite{ncd}
\begin{equation}
    \text{NCD}(x, x') = \frac{|\mathcal{C}(xx')| - \min\{|\mathcal{C}(x)|, |\mathcal{C}(x')|\}}{\max\{|\mathcal{C}(x)|, |\mathcal{C}(x')|\}} + \varepsilon,
    \label{eq:ncd}
\end{equation}
where $|\mathcal{C}(z)|$ is the length of the compressed form of the data, $z$, using compression algorithm, $\mathcal{C}$, the notation $xx'$ denotes the concatenation of strings $x$ and $x'$, and $\varepsilon\geq0$ is an error term accounting for imperfect compression algorithms~\cite{ncd}. The error term is usually assumed to be small relative to the other term.

NCD requires a choice of compression algorithm, and here we evaluated the \texttt{gzip}~\cite{gzip}, \texttt{bz2}~\cite{bz2}, and \texttt{brotli}~\cite{google} compressors.
To distinguish between the use of these different compressors, a subscript is used such that NCD$_{\text{gzip}}$ denotes the use of the \texttt{gzip} compressor.

While NCD is often discussed as a measure of distance~\cite{opitz2023gzip,weinreich2023parameter,nishida2011tweet,jiang2022less,ncd}, it does in fact not adhere to the axioms of metric spaces.
A function, $d:X \times X \rightarrow \mathbb{R}$, associated with a set of points, $X$, where $\mathbb{R}$ denotes the set of real numbers, is said to be a metric if the following four axioms hold for all $x_1, x_2, x_3 \in X$~\cite{metrics}:
\begin{align}
    \text{Zero Axiom:} \quad & d(x_1,x_2) = 0 \iff x_1 = x_2 \label{eq:axiom_zero} \\
    \text{Non-negativity Axiom:} \quad & d(x_1,x_2) \geq 0 \label{eq:axiom_nonnegativity} \\
    \text{Symmetry Axiom:} \quad & d(x_1,x_2) = d(x_2, x_1) \label{eq:axiom_symmetry} \\
    \text{Triangle Inequality:} \quad & d(x_1,x_3) \leq d(x_1,x_2) + d(x_2,x_3) \label{eq:axiom_triangle}.
\end{align}

Much of the literature devoted to NCD has treated it as a proper metric~\cite{opitz2023gzip,weinreich2023parameter,nishida2011tweet,jiang2022less}.
However, as we show now, it is not, which can lead to erroneous classifications or violated assumptions~\cite{kernels}.

\begin{lemma}
    When using \texttt{gzip}, \texttt{bz2}, and \texttt{brotli} compressors, NCD does not adhere to the axioms in Equations~\ref{eq:axiom_zero}--\ref{eq:axiom_triangle}, and is thus not a metric.
    \label{pseudometric}
\end{lemma}
\begin{proof}

    We show in what follows that NCD fails to adhere to the axioms for metrics by counter-examples.

    \vspace{0.5em}
    \noindent%
    \textit{Zero axiom:} The following counter-examples violate the zero axiom:
    \begin{align*}
        \text{NCD}_{\text{gzip}}(A,A) = 0.05,~&~ \text{NCD}_{\text{bz2}}(B,G) = 0, \\
        &\qquad\quad \text{and~~NCD}_{\text{brotli}}(X,X) = 0.2.
    \end{align*}

    \vspace{0.5em}
    \noindent%
    \textit{Non-negativity axiom:} The following counter-examples violate the non-nega\-tivity axiom:
    $$
        \text{NCD}_{\text{gzip}}(AAAA,AAAA) = -0.04,
    $$
    $$
        \text{NCD}_{\text{bz2}}(AABABAA,BAABAAB) = -0.03,
    $$
    and
    $$
        \text{NCD}_{\text{brotli}}(CCCCBBCCC, CBCCCBBCCC) = -0.08.
    $$

    \vspace{0.5em}
    \noindent%
    \textit{Symmetry axiom:} The following counter-examples violate the symmetry axiom:
    $$
        \text{NCD}_{\text{gzip}}(AA, BAA) = 0.13 \neq \text{NCD}_{\text{gzip}}(BAA, AA) = 0.04,
    $$
    $$
        \text{NCD}_{\text{bz2}}(AA, AAB) = 0.11 \neq \text{NCD}_{\text{bz2}}(AAB, AA) = 0.00,
    $$
    and
    \begin{align*}
        \text{NCD}_{\text{brotli}}(& AAAAAAA, B) = 0.6 \\
        & \qquad\qquad\quad \neq \text{NCD}_{\text{brotli}}(B, AAAAAAA) = 0.7.
    \end{align*}

    \vspace{0.5em}
    \noindent%
    \textit{Triangle Inequality:} The following counter-examples violate the triangle inequality:
    \begin{align*}
        \text{NCD}_{\text{gzip}}(& AAA, A) > \\
        &\quad \text{NCD}_{\text{gzip}}(AAA, AAAA) + \text{NCD}_{\text{gzip}}(AAAA, A),
    \end{align*}
    $$
        \text{NCD}_{\text{bz2}}(BC, AN) > \text{NCD}_{\text{bz2}}(BC, J) + \text{NCD}_{\text{bz2}}(J, AN),
    $$
    and if
    \begin{align*}
        x_1 = CAAAACAA,~&~ x_2 = CAC, \\
        & \qquad \text{and}~~x_3 = CCACCCACCC,
    \end{align*}
    then
    $$
        \text{NCD}_{\text{brotli}}(x_1,x_3) > \text{NCD}_{\text{brotli}}(x_1,x_2) + \text{NCD}_{\text{brotli}}(x_2,x_3).
    $$
\end{proof}

\subsection{Other String Metrics}
\label{other_string_metrics}
To model datasets that comprise strings, several existing measures of distance between strings are routinely used~\cite{levenshtein}.
To evaluate the relative performance of the NCD metric, we compared it to several other common measures of string distance.
\textit{Levenshtein} is the ``edit distance'' or minimum number of single-character edits to transform one string into another~\cite{navarro2001guided}.
\textit{Hamming} is the number of character positions where two strings differ.
\textit{Hamming Ratio} is the number of character positions where two strings differ, divided by the length of the longer string (denoted ``Ratio'' in the figures).

\subsection{Calculating the distance matrix}
\label{distance_matrix}

In what follows, the pairwise distances between two sets of samples (denoted $X$ and $X'$) is collected in a a \textit{distance matrix}, $D$.
When computing the value of NCD$(x,x')$, it is necessary to calculate the values of $\mathcal{C}(x)$ and $\mathcal{C}(x')$.
To classify a sample, Jiang \textit{et al.}~\cite{jiang2022less} iterated over all elements in the sets $X$ and $X'$, where $X$ would be a set of samples with known labels and $X'$ one with unknown labels, such that for each $x_i \in X$ the compression $\mathcal{C}(x_i)$ was computed repeatedly for each $x_j' \in X'$.
Because computational time scales linearly with the size of both $X$ and $X'$, the run time is thus $\mathcal{O}(|X|\cdot|X'|)$ for both compute and memory, where $|X|$ is the cardinality of the set $X$.
Clearly, if the number of samples in $X$ and $X'$ are large, then run-time becomes a concern.
We propose some modifications to how the pairwise distances are computed and which pairwise distances are computed to reduce the computational costs and also make NCD behave ``more like a metric'', which is outlined in Section~\ref{improvements}.

\section{Methods}
\label{methods}

This section outlines Jiang's NCD-KNN method~\cite{jiang2022less} and the proposed modifications to NCD before outlining the data and experiments used to verify the efficacy of said modifications.

\subsection{Proposed Modifications to NCD}
\label{improvements}

Jiang \textit{et al.}'s implementation of the NCD-KNN method~\cite{jiang2022less} becomes inefficient because it includes many repeated computations.
To minimise run-time, we first propose a simple modification to pre-compute the compressed length of all input strings and caching them.
When computing the value of NCD$(x,x')$, it is necessary to calculate the values of $\mathcal{C}(x)$ and $\mathcal{C}(x')$ as in Equation~\ref{eq:ncd}.
For each $x \in X$, the $\mathcal{C}(x)$ would be computed repeatedly when computing the pairwise NCD distances between $x\in X$ and the elements in $X'$.
Instead of recomputing the $\mathcal{C}(x)$ repeatedly, it is much more efficient to pre-compute the compressed versions of each element in $X$ and $X'$ only once.
Since compressing the input samples is the most costly part of this computation, this saves substantial run-time.
In addition, because NCD$(x,x')$ can behave strangely when $x=x'$, we also propose a check for this special case and return 0, thus complying with the zero-axiom.

For the purposes of the experiments below, the method proposed by Jiang \textit{et al.}~\cite{jiang2022less} is denoted ``Vanilla'' and computes the pairwise distances between two sets by naively computing every element of a distance matrix using the NCD function in Equation~\ref{eq:ncd}.
Larger modifications to the ``Vanilla'' method discussed by Jiang \textit{et al.}~\cite{jiang2022less} are proposed in the following subsections.
We propose several modifications that symmetrise the distance matrix, intended to ensure adherence to the symmetry axiom (Equation~\ref{eq:axiom_symmetry}), as discussed in Section~\ref{symmetrisation}.
Finally, we outline the proposed way to use NCD as a kernel in Section~\ref{kernels}.

\subsubsection{Symmetrisation}
\label{symmetrisation}

We propose three new ways to induce ``more'' adherence to the axioms in Equations~\ref{eq:axiom_zero}--\ref{eq:axiom_triangle} and compare their efficacy and run-time in Section~\ref{results}.

The second method, proposed here as a modification to NCD, assumes symmetry by only computing the lower triangular part of the distance matrix and then reflecting those values about the diagonal, instead of computing the entire distance matrix; this method is denoted ``Assumed''.
Hence, in a distance matrix, $D$, the ``Assumed'' method computes the lower triangular part of $D$ and then let
\begin{equation}
    D_{i,j} = D_{j,i},
    \label{eq:assumed}
\end{equation}
which effectively halves the computational cost of computing the distance matrix.

In the third method, proposed here, symmetry is \textit{enforced} by sorting the inputs of NCD alphanumerically before computing the distance between them.
This approach thus ensures symmetry during prediction as well as training.
This method is denoted ``Enforced'', and also effectively halves the computational cost of computing the distance matrix since again $D_{i,j} = D_{j,i}$.

The fourth method, also proposed here, computes the \textit{average} value of $\text{NCD}(x,x')$ and $\text{NCD}(x',x)$.
The average of $\text{NCD}(x,x')$ and $\text{NCD}(x',x)$ is
\begin{equation}
    \label{eq:naive_average_symmetrisation}
    \overline{\text{NCD}}(x,x') = \frac{\text{NCD}(x,x') + \text{NCD}(x', x)}{2},
\end{equation}
which can be simplified to
\begin{equation}
    \label{eq:simplified_average_symmetrisation}
    \overline{\text{NCD}}(x, x') = \frac{\frac{\mathcal{C}(xx') + \mathcal{C}(x'x)}{2} - \min[\mathcal{C}(x), \mathcal{C}(x')]}{\max[\mathcal{C}(x), \mathcal{C}(x')]} + \varepsilon,
\end{equation}
which clearly leads to $D_{i,j}=D_{j,i}$.
The simplification in Equation~\ref{eq:simplified_average_symmetrisation} thus only includes one additional compression and only requires the lower-triangular distance matrix to be computed. Hence, the computational cost is 66.67\% of the cost of the Vanilla method, and not a doubling of the computational cost as Equation~\ref{eq:naive_average_symmetrisation} suggests.
This method is denoted ``Average''.

\subsection{Kernelisation}
\label{kernels}

We propose to use NCD to construct an approximate kernel, allowing NCD to be used with a much larger set of ML methods than as a distance.
For this purpose, a kernel is defined as a function, $k : X \times X \rightarrow \mathbb{R}$, such that
\begin{equation}
    k(x, x') := \langle \phi(x), \phi(x') \rangle
    \label{eq:kernel}
\end{equation}
for all $x, x' \in X$,
where $\phi: X \to Y$ is a feature function (a function extracting features from its inputs), and $\langle \cdot, \cdot \rangle$ denotes an inner product in the feature space, $Y$.
The $i$-th row and $j$-th column of a kernel matrix, $K$, is
$$
    K_{i,j} = k(x_i, x_j).
$$

The radial basis function (RBF) kernel~\cite{shalev2014understanding}, also known as the Gaussian kernel (when the distance is the Euclidean distance) is defined as
\begin{equation}
    k(x, x') = \exp\left(-\frac{d{(x, x')}^2}{\lambda}\right),
    \label{eq:rbf_kernel}
\end{equation}
where $\lambda$ is a tunable parameter (denoted a \textit{length scale}) that controls how quickly the kernel function decreases as a function of the distance between points, \textit{i.e.}, determines the influence of individual points on neighbouring points.
We thus propose to use NCD as the distance function, $d$, in the kernel in Equation~\ref{eq:rbf_kernel}.
The RBF kernel is particularly effective as it is known to be a universal function approximator~\cite{rbf_universal}.

The Hamming kernel~\cite{hamming_kernel}, based on the Hamming distance between two strings or binary vectors, is defined as
\begin{equation}
    k(x, x') = 1 - \lambda \frac{d_H(x,x')}{\max(|x|,|x'|)},
    \label{eq:hamming_kernel}
\end{equation}
where $ d_H(x, x') $ denotes the Hamming distance, $\max(|x|,|x'|)$ denotes the length of the longer string, and $ \lambda $ is a tunable parameter that controls the sensitivity of the kernel to differences between input vectors.
Smaller values of $ \lambda $ cause the kernel to decay more rapidly as the number of differing positions increases, thereby emphasizing exact or near-exact matches.
We propose to use this kernel in settings where inputs are strings or binary representations, as it naturally captures similarity through positional agreement.
Like the RBF kernel, the Hamming kernel belongs to the class of positive-definite kernels and has been shown to be effective at classification tasks~\cite{hamming_classification}.

Euclidean distances can be computed in the feature space by using a kernel. We see that
\begin{align*}
    d(x,x')
        &= \| \phi(x) - \phi(x') \|_2^2 \\
        &= \langle\phi(x) - \phi(x'), \phi(x) - \phi(x') \rangle \\
        &= \langle \phi(x), \phi(x) \rangle + \langle \phi(x'), \phi(x') \rangle - 2\langle \phi(x), \phi(x') \rangle \\
        &= k(x, x) + k(x', x') - 2k(x, x'),
\end{align*}
and denote this distance as the \textit{kernel distance}.
For the RBF and Hamming kernels, when the symmetry and the zero axioms hold, we known that $k(x,x) = k(x',x') = 1$, and the kernel distance can be computed efficiently as
\begin{equation}
    d_k(x, x') = 2 - 2 k(x, x').
    \label{eq:kernel_distance}
\end{equation}
This formulation was used to extend NCD-KNN to kernels, \textit{i.e.}, the kernel distance in Equation~\ref{eq:kernel_distance} was used together with the RBF kernel in Equation~\ref{eq:rbf_kernel}, as well as with the Hamming kernel in Equation~\ref{eq:hamming_kernel}.
Logistic regressors and SVCs were trained on the kernel matrices formed from Equations~\ref{eq:rbf_kernel}~\&~\ref{eq:hamming_kernel}.
The KNN models were trained using a matrix calculated with Equation~\ref{eq:kernel_distance} for both the Hamming and RBF kernels.

\subsection{Data}
\label{datasets}

Several open datasets were used to evaluate the efficacy of NCD in the context of heterogeneous tabular and text data.

We used the KDD-NSL data, which is a log of system process data for both regular users (denoted benign) and malicious software (denoted adversarial)~\cite{kddnsl}. It includes 6,072 samples and 41 features that encapsulate the behaviour of both benign software and malware.
KDD-NSL includes software protocol, system error rate, whether the process has root privileges, and the number of files accessed by the process.

We also used the DDoS IoT dataset~\cite{ddos}, which includes information collected from network packet headers of adversarial and benign users across many types of DDoS attacks.
Specific features include source IP address, source port, destination IP address, destination port, and network protocol among a total of 90 features across more than 40 million samples, collected from both benign users and malicious traffic.
We used the Truthseeker dataset~\cite{truthseeker}, which includes 134 thousand messages from Twitter users with a label provided by the data distributors, and a label that encodes whether or not a given user was a suspected bot.
Finally, we used the SMS Spam dataset~\cite{sms_spam} which includes SMS messages and a label indicating whether or not a message is spam across 5,575 samples.

For several of the datasets, malicious examples were rare compared to the number of benign examples. To address the class imbalances, each dataset was under-sampled~\cite{undersampling} using the \texttt{imblearn} package~\cite{imblearn} to reduce bias towards the majority class and to ensure metrics like accuracy are meaningful.

In an effort to represent both text and numerical data as strings, the rows of each numerical dataset were extracted as lists in \texttt{Python} and then those lists were cast directly to strings for the DDoS, KDD-NSL, and SMS Spam datasets.
This crude conversion was intentionally adopted as a baseline representation, allowing us to evaluate model performance without introducing additional preprocessing heuristics or feature engineering that might confound comparisons across disparate datasets.

\subsection{Experiments}
\label{experiments}
We evaluated the proposed methodology using the described datasets, models, symmetrisation methods (``Vanilla'', ``Assumed'', ``Enforced'', and ``Average''), and metrics ($\text{NCD}_{\text{gzip}}$, $\text{NCD}_{\text{bz2}}$, and $\text{NCD}_{\text{brotli}}$, Levenshtein distance, Hamming distance, and a normalised Hamming distance (labelled ``Ratio'')).
After generating the 5-fold cross-validation sets for each dataset-metric-symmetrisation combination, the distance matrices for each dataset-metric-symmetrisation combination method were computed and cached, as outlined in Sections~\ref{distance_matrix} and~\ref{improvements}.
Additionally, kernel matrices were computed as described in Section~\ref{kernels}.
The classifiers used were KNN, logistic regression, and SVC, as implemented in \texttt{scikit-learn}~\cite{sklearn}.

Each model was tuned using the hyper-parameters outlined in the following.
Both the RBF and Hamming kernels have a hyper-parameter, $\lambda$, that had to be tuned; $\lambda$ was was evaluated in powers of 10 in the range $[10^{-3}, 10^3]$.
KNN requires the model builder to specify the number of nearest neighbours.
We evaluated $k \in \{1,3,5,7,11\}$, as odd numbers means there were no ties (for the binary classification task).
In logistic regression, an $\ell_2$ penalty was used as well as a configuration without any penalty.
The coefficient of the penalty was set to powers of 10 in the range $[10^{-3}, 10^3]$.
The \texttt{SAGA} solver~\cite{saga} was used for logistic regression, with a tolerance of $10^{-4}$.
The penalty term in the SVC was varied in the range $[10^{-3}, 10^3]$ for each power of ten.

To find the most appropriate set of hyper-parameters, each of the dataset-model-symmetrisation-metric combinations enumerated above were evaluated using a grid search and 5-fold cross-validation on the training data.
However, the accuracies and times reported below are calculated on 200 \textit{test} samples withheld from the training and cross-validation processes for the best-fit model for each dataset-model-symmetrisation-metric combination.

All experiments were run on an Apple M4 Pro with 12 cores rather than a data-center CPU so that timing measurements reflect (an admittedly high-end) client device.

\section{Results and Discussion}
\label{results}

In this section, the results from the aforementioned experiments are discussed.

\begin{figure*}[!htb]
    \centering
    \includegraphics[width=.30\textwidth]{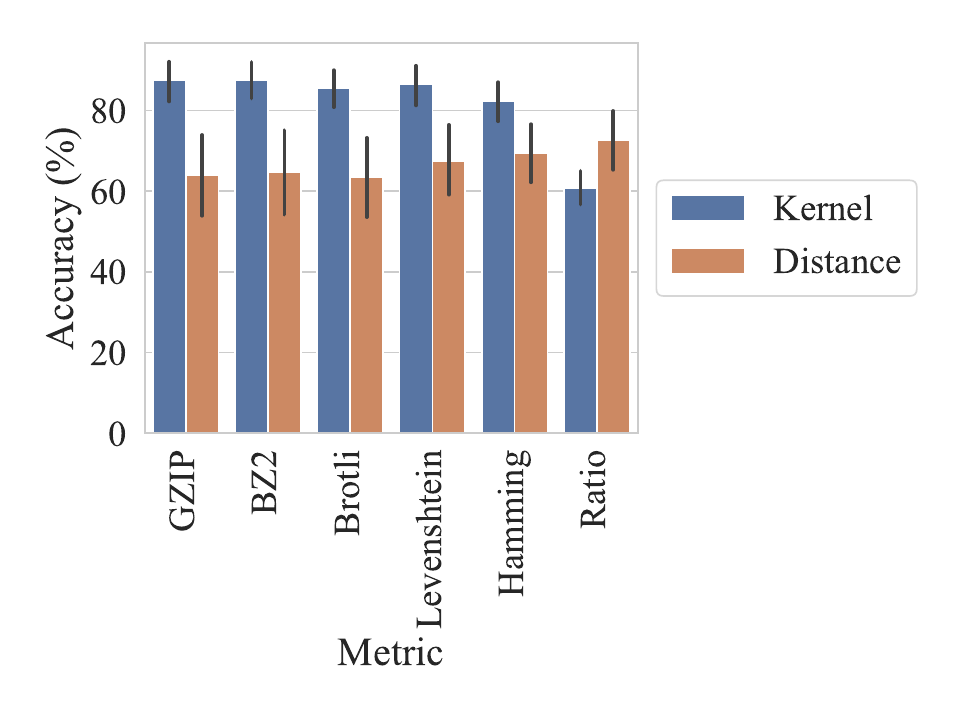}
    \includegraphics[width=.30\textwidth]{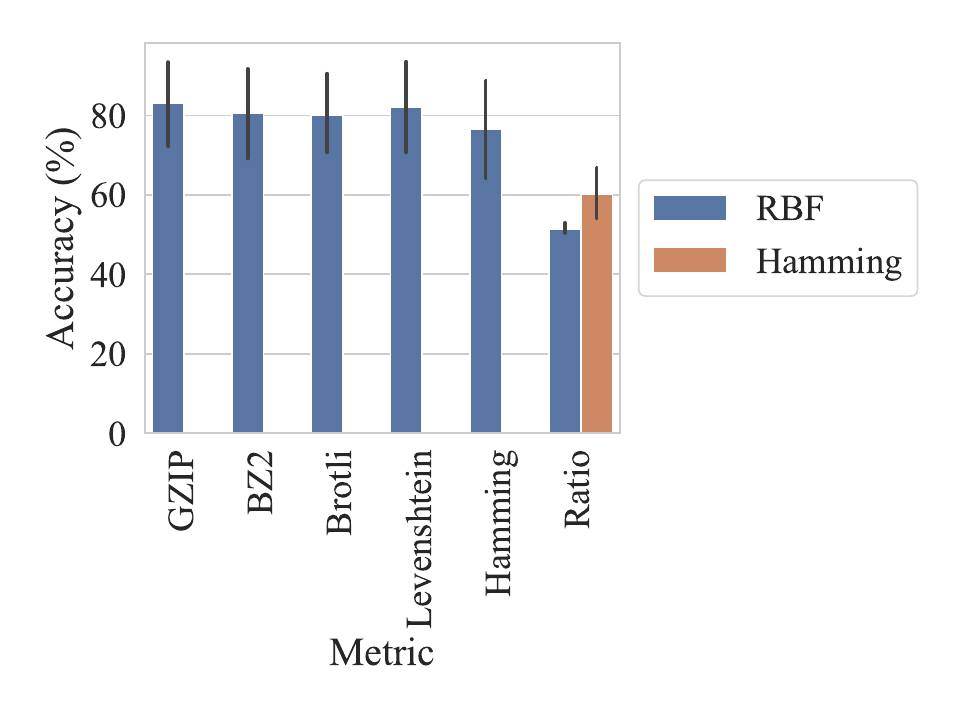}
    \includegraphics[width=0.30\textwidth]{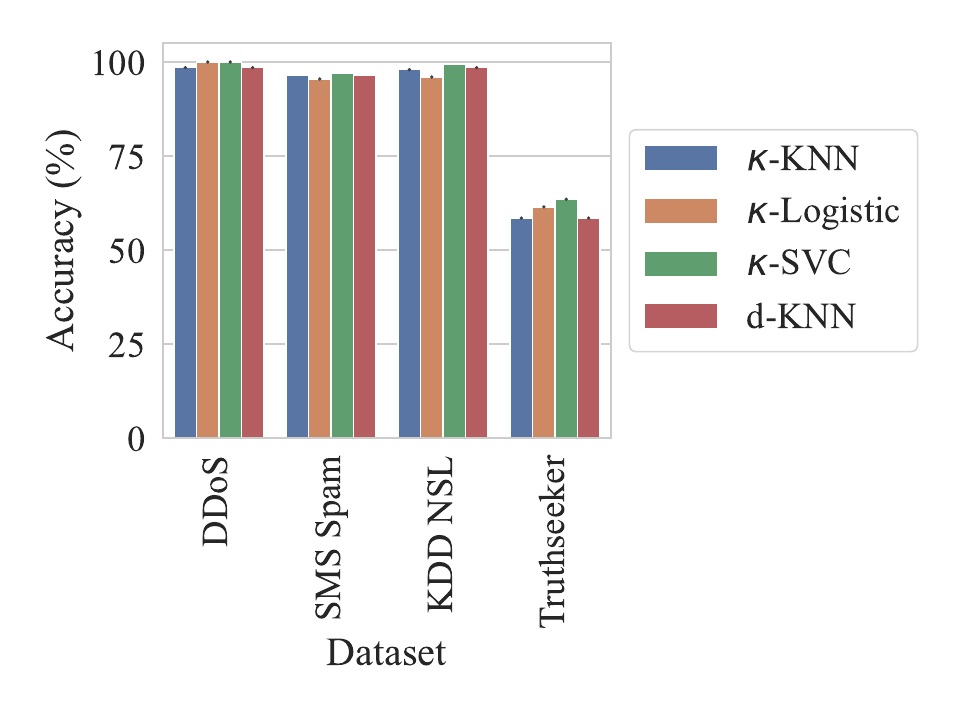}
    \caption{
    This figure depicts the accuracy broken down by string metric (left and centre plots) as well as by dataset and model (right plot).
    The left and centre plots depict the test accuracy for all best-fit dataset-model-symmetrisation-metric combinations with the error bar indicating the 95\% interval.
    The left plot compares the performance between models that use kernel methods (blue) and models that use distance alone (orange, KNN only).
    The centre plot compares the RBF kernel  (blue) to the Hamming kernel (orange)---the latter is only defined for the ``Ratio'' metric.
    In contrast, the right plot depicts test accuracy on \textit{only} the best NCD-kernelised model (denoted with a ``$\kappa$-'' prefix) and the best distance-based KNN proposed by Jiang \textit{et.~al.}~\cite{jiang2022less} (denoted with a ``$d$-'' prefix).
    }
    \label{fig:baseline}
\end{figure*}

Figure~\ref{fig:baseline} depicts the test accuracy across all dataset-metric-model-kernel combinations.
The left and centre subplots of Figure~\ref{fig:baseline} depict the test accuracies for each best-fit dataset-metric-model-kernel-symmetrisation combination, and  the error bars indicate the 95\% confidence interval.
The left subplot of Figure~\ref{fig:baseline} clearly indicates that kernel methods (blue) tend to outperform distance-based methods when using NCD with various compressors (orange).
Furthermore, the left subplot shows that that kernelised NCD is often significantly better than string metrics or their kernelised counterparts.
The centre subplot of the same figure clearly shows that for all metrics besides ``Ratio'' (the only metric for which the Hamming kernel is defined), that the RBF kernel significantly outperforms the baseline Hamming kernel.
Additionally, the centre subplot shows that NCD is significantly more accurate than any of the baseline string metrics since compressors encode more semantic and string frequency data that metrics alone.
The right subplot validates the performance of the proposed method by depicting the test accuracies for the best-fit model using kernelised NCD, clearly indicating that kernelized NCD models are often more accurate than NCD-KNN\@.

\begin{figure*}[!htb]
    \centering
    \includegraphics[width=0.60\textwidth]{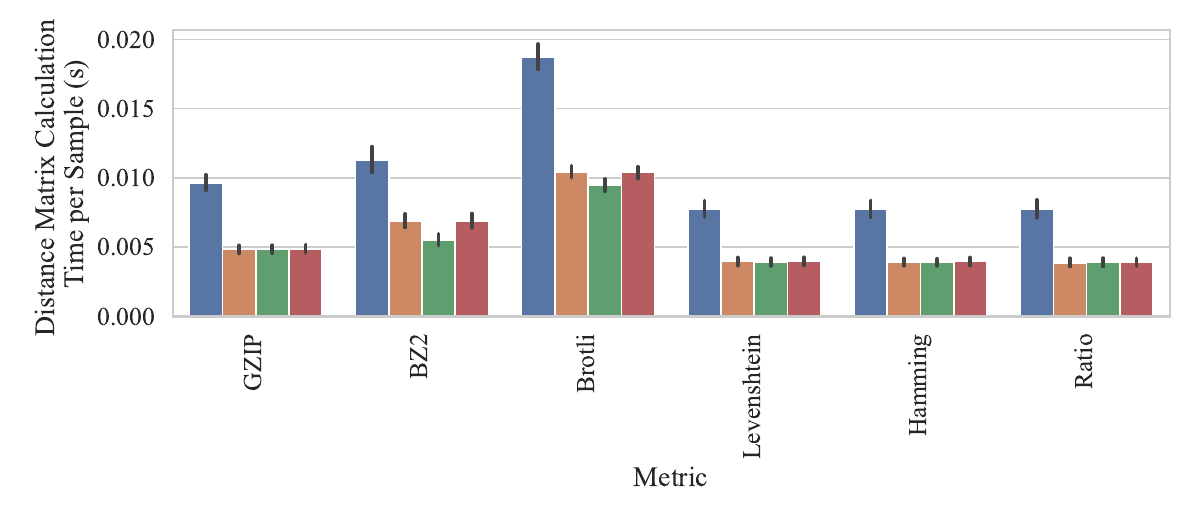}
    \includegraphics[width=0.35\textwidth]{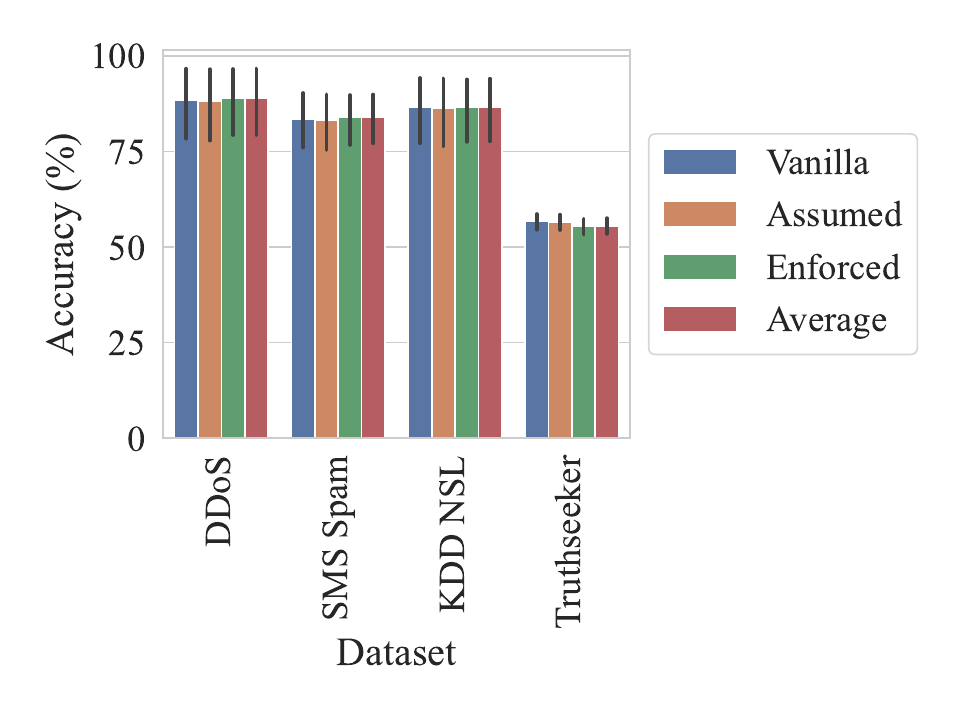}
    \caption{
    The distance calculation time (left) and accuracy (right) across each dataset (columns), distance metric (first-axis, left plot),  dataset (first-axis, right plot), model (left plot, colour), and symmetrisation method (right plot, colour).
    For both plots, the bars represent the mean and the error bars represent 95\% confidence intervals and colour represents the symmetrisation method.
    The left plot was calculated using 1000 training samples for each metric-algorithm combination.
    The right plot was calculated using the test accuracies and 95\% confidence interval for the best fit NCD-kernelised model, calculated for each of the symmetrisation-data-model-compressor combination.
    }
    \label{fig:acc_and_time}
\end{figure*}

Figure~\ref{fig:acc_and_time} shows the classifier performance across each dataset and string metric in terms of run-time (left) and accuracy (right), with the error bar indicating the 95\% performance for all best-fit dataset-metric-model-kernel-symmetrisation combinations.
It is clear from the right subplot of Figure~\ref{fig:acc_and_time} that choice of symmetrisation method has a substantial effect on run-time, with the ``Assumed'', ``Enforced'', and ``Average'' symmetrisation methods taking roughly half as long per sample as the ``Vanilla'' version.
It is clear from the left subplot of Figure~\ref{fig:acc_and_time} that the ``Assumed'', ``Enforced'', and ``Average'' symmetrisation methods proposed in this work are superior to those found in the literature by decreasing the run-time without significantly penalising accuracy (Figure~\ref{fig:acc_and_time}, right).

\section{Limitations}
\label{limitations}
While the methods presented in this work demonstrate strong performance across various tasks and datasets, several limitations should be acknowledged.
To further improve run-time performance, compression algorithms optimised for graphics processing units (GPUs) have been developed~\cite{gpu_compression}.
These are likely to outperform the CPU-based implementations used in this paper when applied to large-scale datasets.
Incorporating such GPU-accelerated compressors could significantly reduce processing time and improve scalability.
This work focused on a specific set of compression algorithms.
Although other compressors exist---including those optimised for specific data types~\cite{mp3,hevc}---they were considered out of scope for this study.
Likewise, the data preprocessing step used for heterogeneous tabular data was intentionally kept simple---each row was cast to a Python list then cast again as a string.
While effective for our purposes (see Figure~\ref{fig:baseline}), this approach is crude and may obscure structural or semantic relationships within the data by including extraneous characters like white spaces, commas, and brackets (used to delineate and denote a list in Python, respectively).
More sophisticated encoding methods—such as schema-aware parsing or embedding techniques—could improve performance, particularly on more complex or high-dimensional datasets, in particular for the ``Truthseeker'' dataset~\cite{arslan2025prediction}.

\section{Conclusion}
\label{conclusion}

Overall, we see that NCD is at least as accurate as other string metrics (left subplot of Figure~\ref{fig:baseline}), despite not being a true metric (Lemma~\ref{pseudometric}).
Furthermore, we see that the proposed symmetrisation methods are quite effective---sometimes even outperforming the ``Vanilla'' method with respect to accuracy(right subplot of Figure~\ref{fig:acc_and_time}).
Additional run-time improvements are significant and obvious (left subplot of Figure~\ref{fig:acc_and_time}).

Furthermore, because NCD is known to work well on a small number of samples~\cite{ncd}, each model can be unique to each user, session, or device by using data from each user in isolation.
The proposed model is a real-time, client-side classification method that can be trained quickly---potentially on data collected from only a single user and stored only on their device.
The proposed client-side classifier limits the attack surface to only adversaries that have access to data that users generally keep private (\textit{e.g.}, the contents of a message, system data, or network traffic).
The end result is a model with a marginal attack surface that is nevertheless accurate, simple, generally applicable, and fast enough to be trained entirely on a generic client device.

\bibliographystyle{elsarticle-num}
\bibliography{bibliography}

\end{document}